%% file: main.tex
\pdfoutput=1
\documentclass{article}
\usepackage{iclr2022_conference,times}

\usepackage[utf8]{inputenc}
\input{header}
\usepackage[margin=1in]{geometry}
\input{math_commands}

\usepackage{colortbl}
\usepackage{algpseudocode}

\newcommand{\BB}{\text{BERT}_{\text{BASE}}}
\newcommand{\BL}{\text{BERT}_{\text{LARGE}}}
\newcommand{\TB}{\text{T5}_{\text{BASE}}}
\newcommand{\TL}{\text{T5}_{\text{LARGE}}}

\title{Leveraging redundancy in attention \\with Reuse Transformers}

\author{%
Srinadh Bhojanapalli\parbox{0pt}{$^{*}$},  Ayan Chakrabarti\thanks{The first two authors contributed equally.},  %
Andreas Veit,  Michal Lukasik,  Himanshu Jain,\\%
\bf{}Frederick Liu, Yin-Wen Chang \& Sanjiv Kumar\vspace{0.15em}\\%
Google Research\vspace{0.15em}\\%
\texttt{\footnotesize \{bsrinadh,ayanchakrab,aveit,mlukasik,himj,frederickliu,yinwen,sanjivk\}@google.com}
}

\iclrfinalcopy
\begin{document}

\maketitle
\input{figures}
\input{tables}
\input{1_abstract}
\input{2_intro}
\input{3_related}
\input{4_analysis}
\input{5_model}
\input{6_experiments}

\section{Conclusion}
In this paper, we analyzed the similarity in attention scores computed at different layers of a Transformer model, and discovered them to be substantially redundant. Based on this observation, we proposed a new approach for reducing the compute and memory usage of Transformer models, both during training and inference, by reusing attention scores across layers. As our extensive experiments showed, this improved efficiency was borne out in terms of actual training speed and memory usage, and came with performance equivalent to or better than standard Transformers. More broadly, our work shows that developing a better understanding of the empirical behavior of state-of-the-art models can yield real dividends---in this case, in the form of a new architecture with improved performance-efficiency trade-offs.

\section{Reproducibility Statement}
All the experiments in this paper are done with models that have publicly available code. We have used the same default hyperparameters for all our experiments for a given task. We further list them in Appendix~\ref{sec:appx_setup}. Proposed reuse attention (Algorithm~\ref{alg:reuse}) is a simple modification to standard attention layers. 
\bibliography{references}
\bibliographystyle{iclr2022_conference}

\clearpage
\appendix
\begin{center}
{\large \bf Appendix}
\end{center}
\input{7_appendix}
\input{8_analysis}

\end{document}

%% file: header.tex
\usepackage{url}

\usepackage{graphicx} 
\usepackage{subfigure} 
\usepackage{multirow}
\usepackage{xcolor}
\usepackage{mathtools}
\usepackage{relsize}
\usepackage{nicefrac}
\usepackage{xspace}
\usepackage{amsmath,amssymb,enumerate}
\usepackage{amsthm,cancel}
\usepackage{natbib}
\usepackage{enumitem}
\usepackage{dsfont}
\usepackage[colorlinks=true,citecolor=blue]{hyperref}
\usepackage[utf8]{inputenc} 
\usepackage[T1]{fontenc}    

\newtheorem{lemma}{Lemma}

\newtheorem*{lemma*}{Lemma}
\newtheorem*{theorem*}{Theorem}
\newtheorem*{assumption*}{Assumption}
\newtheorem*{corollary*}{Corollary}
\newtheorem*{remark*}{Remark}
\newtheorem*{definition*}{Definition}

\usepackage{hyperref}
\usepackage{url}

\usepackage[utf8]{inputenc} 
\usepackage[T1]{fontenc}    
\usepackage{url}            
\usepackage{booktabs}       
\usepackage{amsfonts}       
\usepackage{nicefrac}       
\usepackage{microtype}      
\usepackage{algorithm}

%% file: math_commands.tex
\usepackage{amsmath,amsfonts,bm}

\def\eqref#1{equation~\ref{#1}}

\def\1{\bm{1}}

\def\mA{{\bm{A}}}

\def\mR{{\bm{R}}}

\def\mW{{\bm{W}}}
\def\mX{{\bm{X}}}
\def\mY{{\bm{Y}}}
\def\mZ{{\bm{Z}}}

\DeclareMathAlphabet{\mathsfit}{\encodingdefault}{\sfdefault}{m}{sl}
\SetMathAlphabet{\mathsfit}{bold}{\encodingdefault}{\sfdefault}{bx}{n}

\def\gS{{\mathcal{S}}}

\newcommand{\E}{\mathbb{E}}

\newcommand{\R}{\mathbb{R}}



%% file: figures.tex
\newcommand{\insertFigReuse}{
\begin{figure*}[!t]
    \centering
    \includegraphics[width=0.8\textwidth]{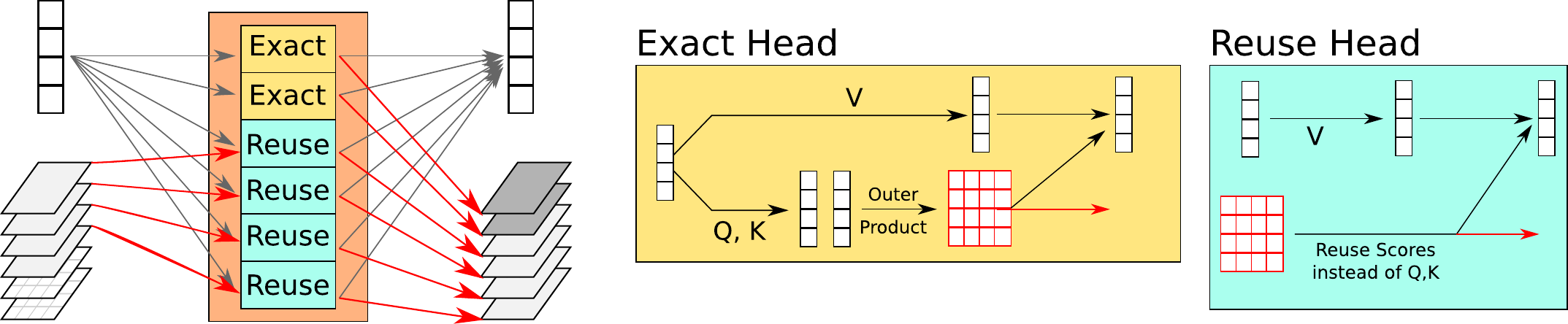}
\caption{\textbf{Reuse Transformer.} We propose a modified architecture for Transformer layers that features a mix of standard ``exact'' and ``reuse'' attention heads, where the latter borrow attention scores computed in previous layers.}\label{fig:reuse}
\end{figure*}
}

\newcommand{\insertFigBestSimilarityModels}{
\begin{figure*}[!t]
    \centering
    \includegraphics[width=0.67\textwidth]{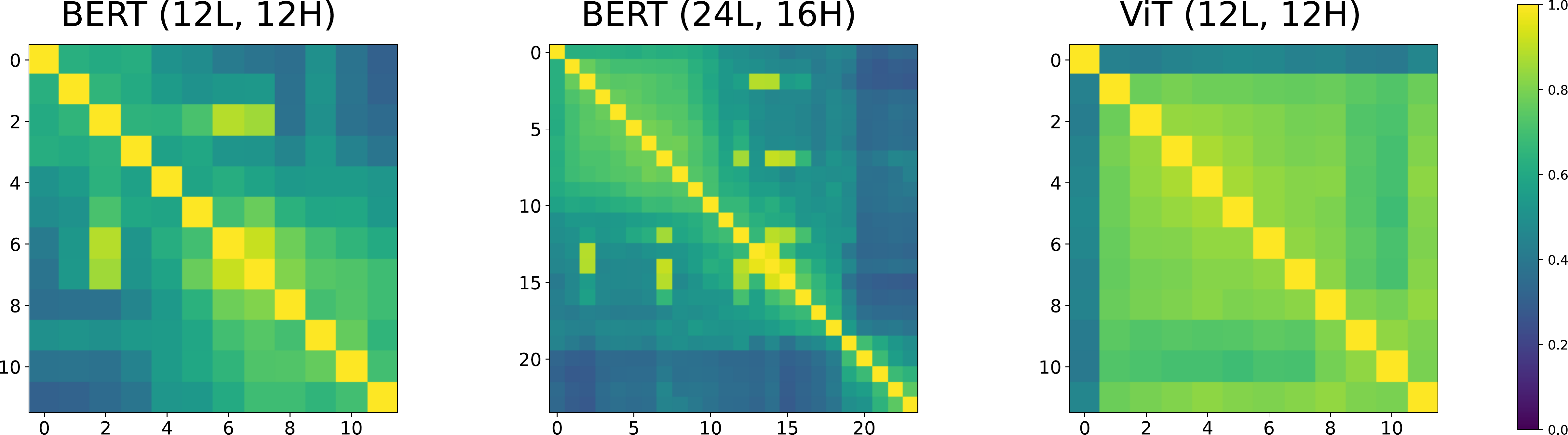}
 \caption{\textbf{All Pairs Similarity.} For all pairs of layers (indexed in x- and y- axis), we visualize the similarity between the best pairs of heads in that pair of layers. We show similarity scores for two BERT models on the Wikipedia dataset and one ViT model on the ImageNet dataset, using scores averaged over 10k examples in all cases.}\label{fig:best_heads_model}
\end{figure*}
}

\newcommand{\insertFigBestSimilarityRandom}{
\begin{figure*}[!t]
    \centering
    \includegraphics[width=0.8\textwidth]{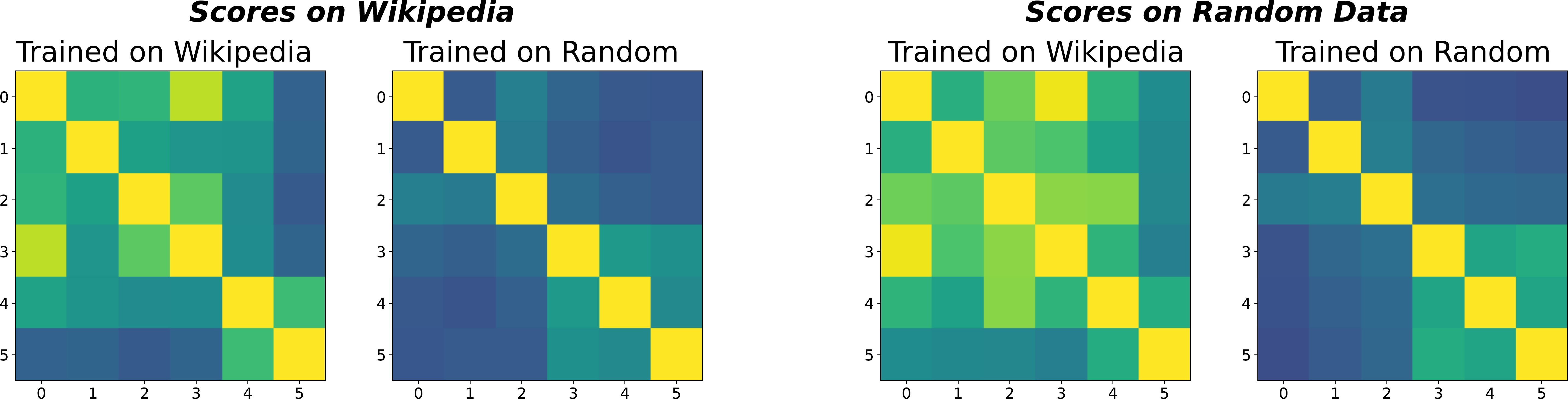}
\caption{\textbf{Role of problem domain.} We visualize the all pairs best head attention similarity scores for (6 layer) Transformer models trained on Wikipedia vs random data. In the first two columns, we compare similarity in attention scores for both models computed over 10k examples from Wikipedia, and in the last two columns over random data. We notice that regardless of which examples are used at test time to compute similarities, the model trained on natural rather than random data exhibits higher similarity.}\label{fig:best_heads_random}
\end{figure*}
}

\newcommand{\insertFigHeadsBertLarge}{
\begin{figure*}[!t]
    \centering
    \includegraphics[width=0.67\textwidth]{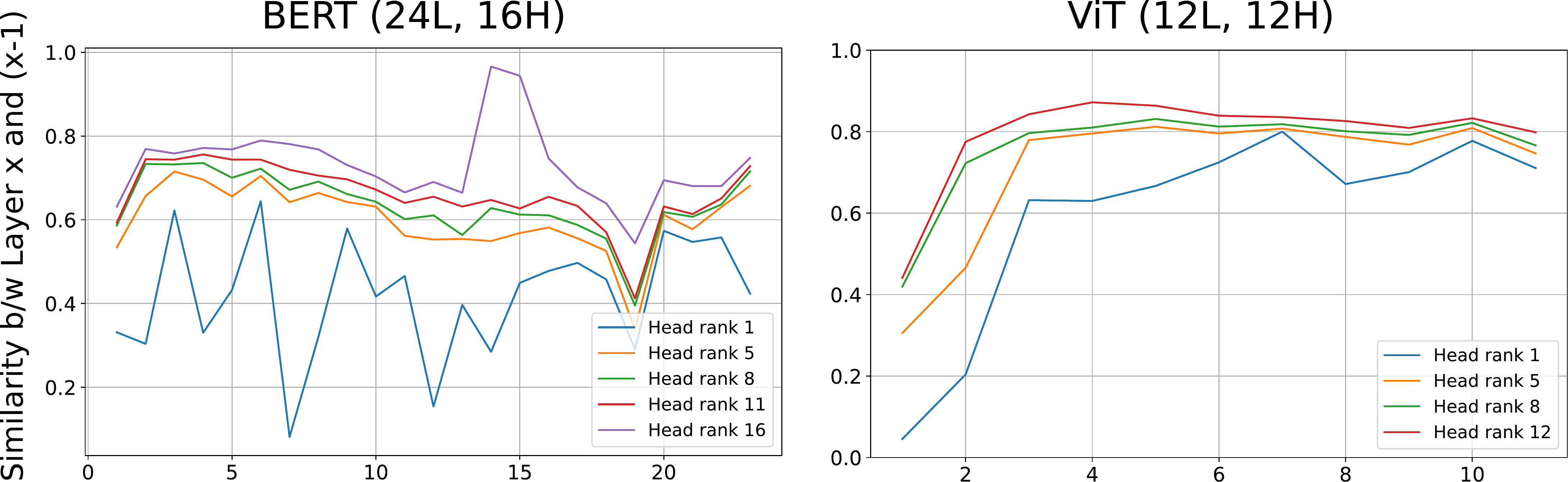}
\caption{\textbf{Sequential Similarity}. We report similairty in adjacent layers for different heads in a (left) 24 layer BERT and (right) 12 layer ViT model. For each layer, we compute similarity for each head with its closest matching head in the previous layer. We then rank heads from lowest (rank 1) to highest similarity, and plot these across layers.}\label{fig:heads_bert_large}
\end{figure*}
}

\newcommand{\insertFigBestSimilarityExamples}{
\begin{figure*}[!t]
    \centering
    \includegraphics[width=0.5\textwidth]{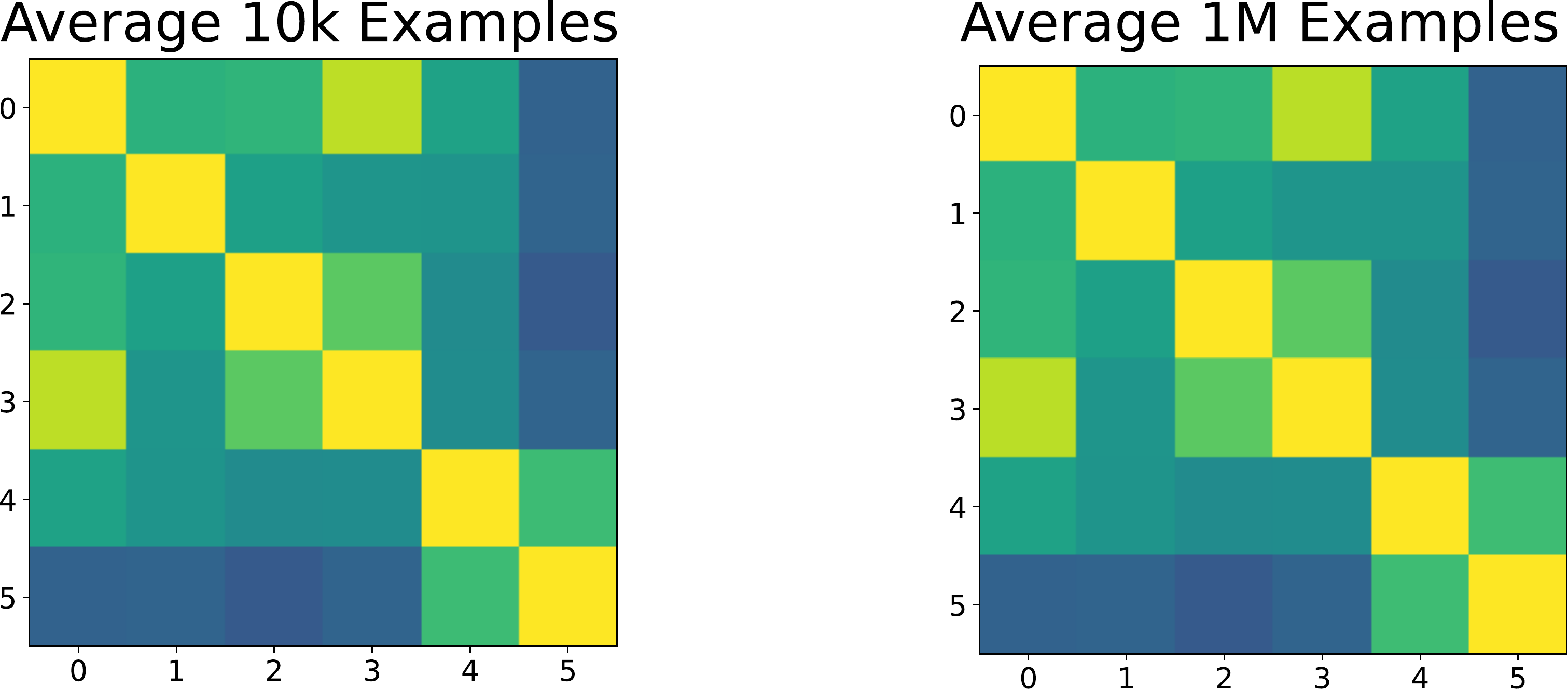}
\caption{\textbf{Number of examples.} We show the all layer pair best head similarity scores for a six layer BERT model, averaged over (left) 10k examples as in the main paper, and over (right) 1M examples. The two averages are essentially identical, showing that our setting of using 10k examples is sufficient to draw conclusions about attention redundancy.}\label{fig:best_heads_6l_examples}
\end{figure*}
}

\newcommand{\insertFigAblation}{
\begin{figure*}[!ht]
    \centering
    \subfigure{
    \includegraphics[scale=0.33]{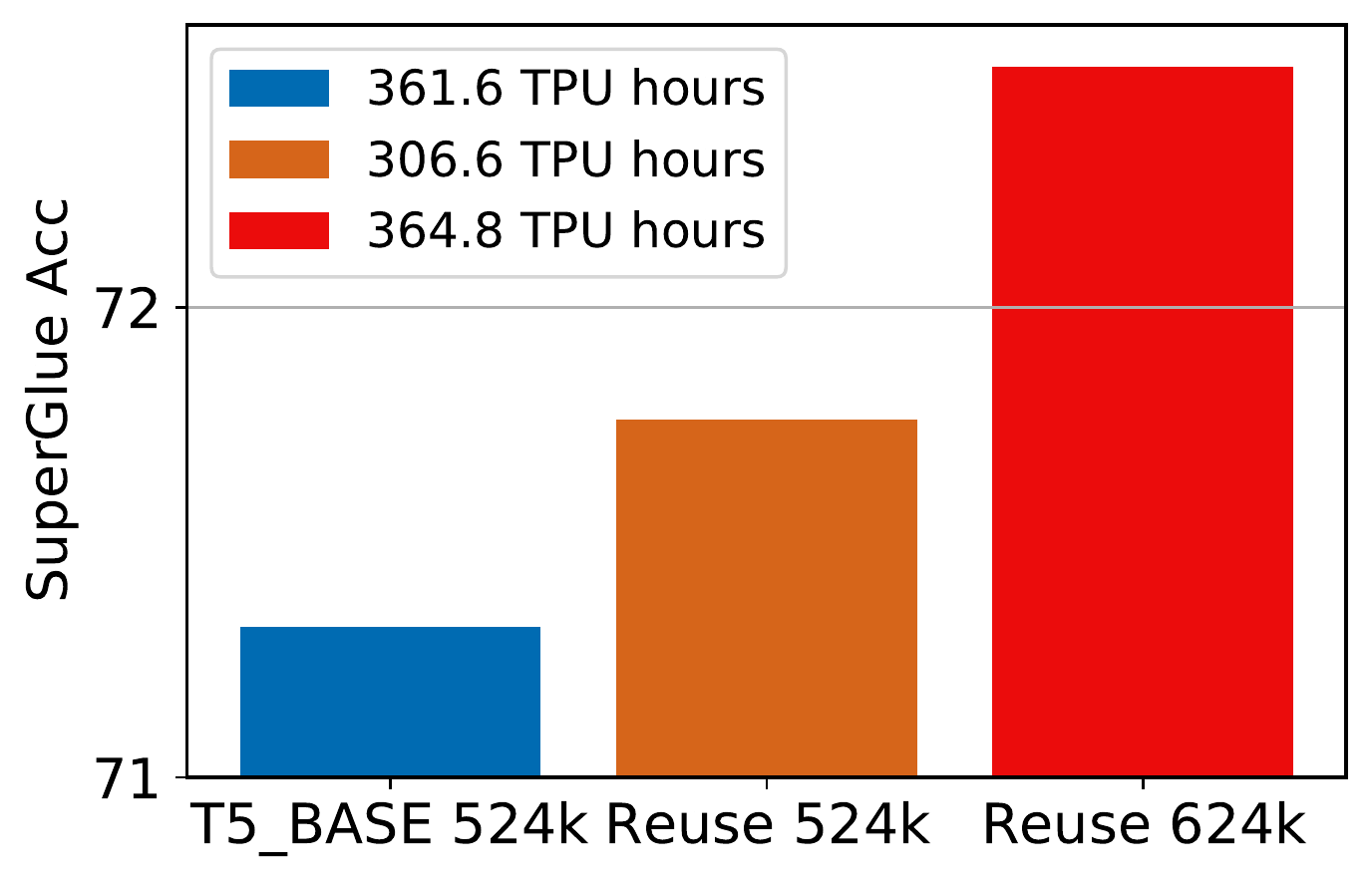}
    }
    \subfigure{
    \includegraphics[scale=0.3]{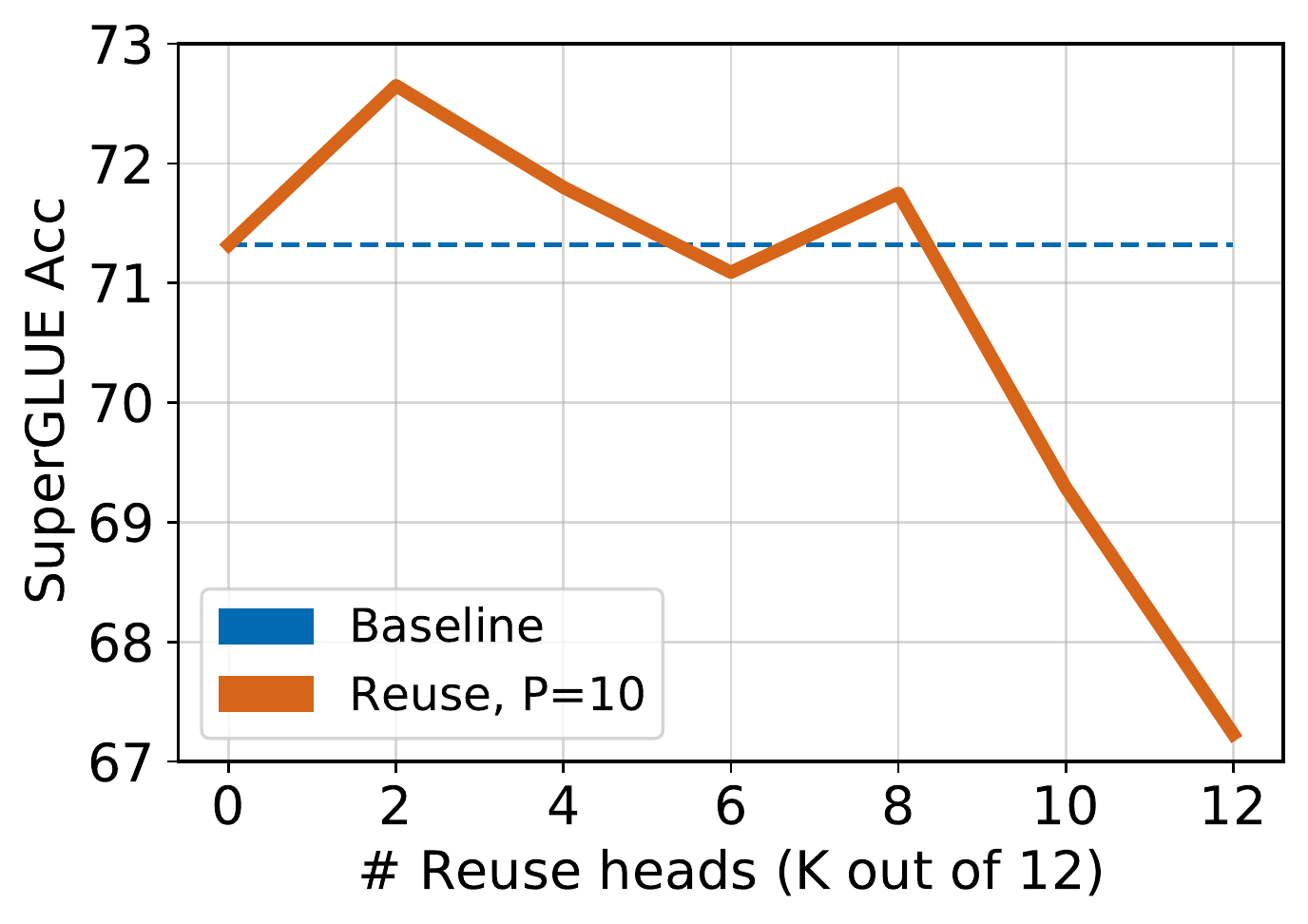}%
    }
    \subfigure{
    \includegraphics[scale=0.3]{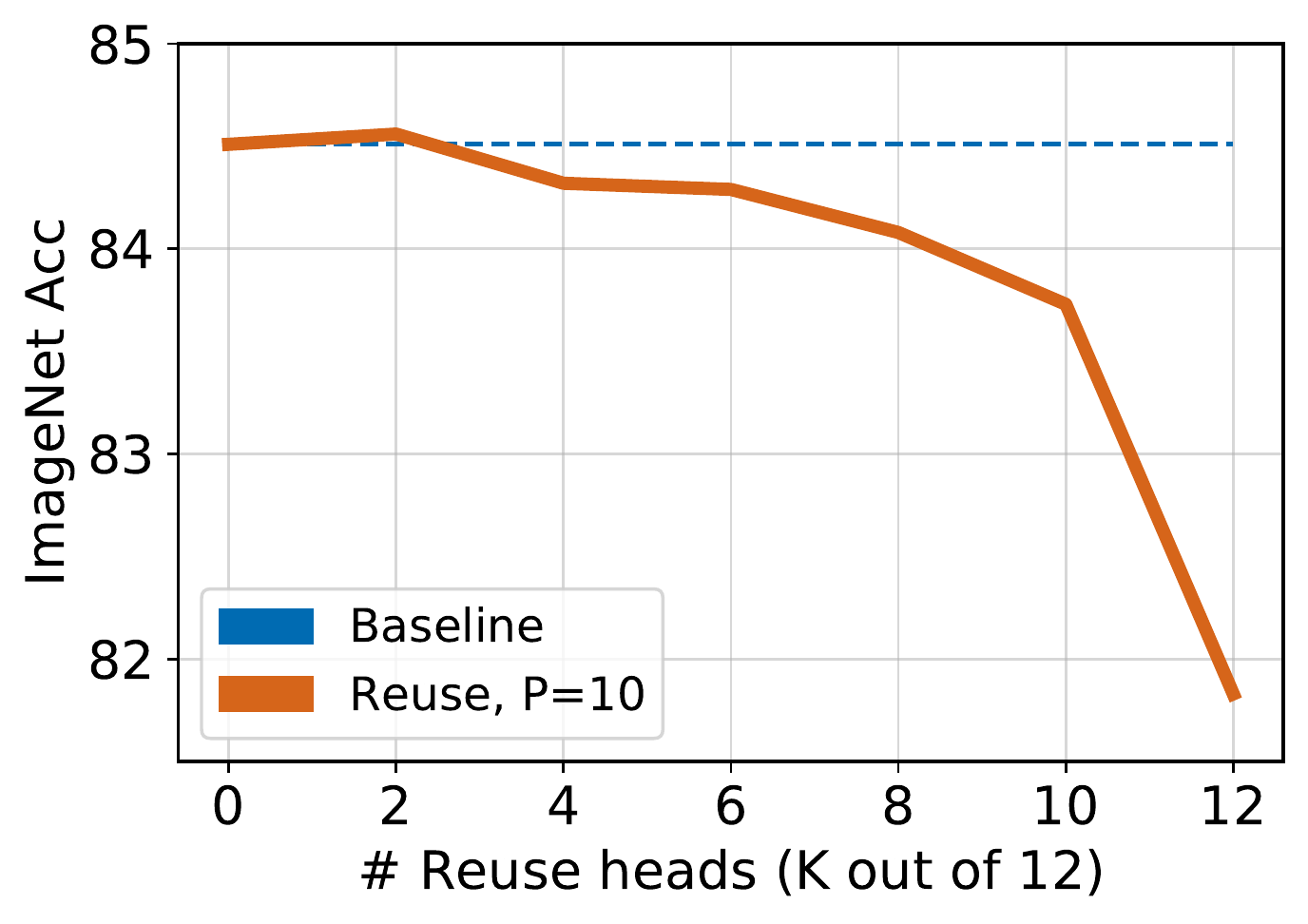}%
    }
\caption{\textbf{Ablation.} Left - Training reuse models for the same amount of time as the $\TB$ results in $>1\%$ improvement on the SuperGLUE benchmark. Middle, Right - Performance of the $\TB$ and ViT models with varying the number of reuse heads on SuperGLUE and ImageNet respectively. We see matching/improved performance with reusing a few heads. Interestingly performance doesn't drop much for ViT models even when $K=12$, inline with the high attention similarity among its layers.}\label{fig:ablation}
\end{figure*}
}

%% file: tables.tex
\newcommand{\CC}[1]{\cellcolor{lightgray}#1}

\newcommand{\insertTableBert}{
\begin{table}[ht]
  \caption{\textbf{BERT.} Median pretraining and finetuning performance of BERT models from 3 independent runs. We highlight reuse Transformer cells that match (upto standard deviation) or improve over the baseline. We notice that reusing attention scores results in similar finetuning performance while reducing the computational requirement. We notice that both forms of reusing attention results in similar performance. We also see that matching the baseline in number of parameters (by increasing number of layers from 12/24 to 13/26 for the BASE/LARGE models) results in better performance showing that reusing attention results in better performance scaling.}
  \label{tab:bert_results}
  \centering
  \small
  \begin{tabular}{lcccccccc}
    \toprule
     \bf Model & \bf Reuse & \bf Reuse  &\bf Params &\bf FLOPS & \bf MLM  & \bf MNLI & \bf SQuAD V1.1 &\bf SQuAD V2.0 \\
     & heads ($K$) & layers ($P$) & & & Acc & Acc & F1& F1 \\ \midrule
     $\BB$ - 12H,12L & - & - & 1.0 & 1.0 & {68.91} & $85.32\pm0.05$ & $89.93\pm0.12$ & $79.53\pm0.71$ \\
     Reuse & 6 & 10 & {0.95} & 0.92 & \CC{68.94} & $84.79\pm0.31$ & $89.29\pm0.27$ & $78.49\pm0.01$ \\
     Reuse & 12 & 6 & 0.94 & 0.9 & \CC{68.97} & \CC{$85.27\pm0.2$} & $89.61\pm0.06$ & $78.74\pm0.25$ \\
     Reuse 13L & 12 & 6 & 1.0 & 0.98 & \CC{69.32} & \CC{$85.38\pm0.32$} & \CC{$90.26\pm0.32$} & \CC{$79.55\pm0.62$} \\
     \midrule
     $\BL$ - 16H,24L & - & - & 1.0 & 1.0 & 73.76 & $87.97\pm0.32$ & $91.87\pm0.17$ & $82.41\pm0.38$ \\
     Reuse & 8 & 22 & 0.93 & 0.91 & \CC{73.93} & \CC{$87.6\pm0.26$} & \CC{$91.78\pm0.12$} & \CC{$83.58\pm0.67$} \\
     Reuse & 16 & 12 & 0.92 & 0.9 & 73.64 & \CC{$87.75\pm0.38$} & \CC{$91.92\pm0.33$} & \CC{$82.56\pm0.14$} \\
     Reuse 26L & 16 & 12 & 1.0 & 0.99 & \CC{74.13} & \CC{$88.14\pm0.10$} & \CC{$92.37\pm0.10$} & \CC{$83.38\pm0.35$} \\
    \bottomrule
  \end{tabular}
\end{table}
}

\newcommand{\insertTableTfive}{
\begin{table}[ht]
  \caption{\textbf{T5.} Median performance of T5 models on the GLUE and SuperGLUE finetuning tasks over 3 independent runs. We highlight reuse Transformer cells that match or improve over the baseline. We notice that reusing attention scores leads to an improvement in performance for both base and large models. We also report the relative number of parameters and compute required for all the models. We notice that reusing attention also leads to an improvement in compute and reduction of model parameters.}
  \label{tab:tfive_results}
  \centering
  \begin{tabular}{lccccccc}
    \toprule
     \bf Model & \bf Reuse & \bf Reuse & \bf Params & \bf FLOPS & \bf Steps/Sec & \bf Glue & \bf SuperGlue\\ 
     &  heads ($K$) & layers ($P$)& & & & Average & Average \\
     \midrule
     $\TB$ - 12H,12L & - & -& 1.0 & 1.0 &12.85 & 84.28$\pm$0.15  & 71.32$\pm$0.28\\
     Reuse & 6 & 10 & 0.92 & 0.9 & {13.7} & \CC{84.66$\pm$0.36} & \CC{71.09$\pm$0.15} \\
     Reuse & 12 & 6 & 0.9 & 0.88 &{15.18} & \CC{84.82$\pm$0.17} & \CC{71.76$\pm$0.11} \\
     Reuse 13L & 12 & 6 & 1.0 & 0.96 & {14.12} & \CC{84.47$\pm$0.13} & \CC{72.14$\pm$0.55}\\     \midrule
     $\TL$ - 16H,24L & - & - & 1.0 & 1.0 & 5.68 & 85.28$\pm$0.28 & 74.16$\pm$1.56\\
     Reuse & 16 & 12 & 0.9 & 0.88& {6.65} & \CC{85.52$\pm$0.23} & \CC{73.44$\pm$0.23}\\
     Reuse 26L & 16 & 12 & 1.0 & 0.92& {6.16} & \CC{85.65$\pm$0.23} & \CC{73.81$\pm$1.41}\\
    \bottomrule
  \end{tabular}
\end{table}
}

\newcommand{\insertTableLRA}{
\begin{table}[ht]
  \caption{\textbf{LRA.} Performance of reuse attention models on the Long Range Arena benchmark. We notice that reusing attention leads to sizeable performance improvement while reducing the computational cost.}
  \label{tab:lra_results}
  \centering
  \small
  \begin{tabular}{lccccccc}
    \toprule
     \bf Model & \bf Reuse & \bf Reuse  &  \bf Avg & \bf ListOps & \bf Text	& \bf Retrieval	& \bf Path \\ 
     &  heads ($K$) & layers ($P$)& Acc  & & & &\\ \midrule
     Baseline - 8H,4L & - & - & 58.84$\pm$0.41 & 36.65 & 63.29 & {58.91} & 76.49 \\
     Reuse & 4 & 2 &   \CC{59.47$\pm$0.71} & \CC{37.4} & \CC{64.07} & 58.08  & \CC{78.33} \\
     Reuse & 8 & 2 &  \CC{60.17$\pm$0.46} & \CC{40.2} & \CC{64.07} & {58.08} &  \CC{78.33} \\
     Reuse & 8 & 3 &  \CC{61.26$\pm$1.48} & \CC{39.95} & \CC{63.95} & {58.72} & \CC{82.4} \\
     \bottomrule
  \end{tabular}
\end{table}
}

\newcommand{\insertTableLRAPerformance}{
\begin{table}[!ht]
  \caption{\textbf{Computational savings.} Performance benchmark on Text classification task in the LRA benchmark for input sequence lengths from 1k to 4k. We see that reusing attention scores leads to significant gains both in increased steps per second and reduced memory usage. We indicate the percentage improvement for the 4k input sequence length setting in the parenthesis.}
  \label{tab:lra_performance}
    \small
  \centering
  \begin{tabular}{lcc|cccc|cccc}
    \toprule
     & \bf Reuse& \bf Reuse& \multicolumn{4}{|c|}{\bf Steps/Second} & \multicolumn{4}{|c}{ \bf Peak Mem Usage (GB)} \\
    \bf Model & heads ($K$) & layers ($P$)  & \bf 1K & \bf 2K & \bf 3K & \bf 4K & \bf 1K & \bf 2K  & \bf 3K & \bf 4K \\     \midrule
    Baseline - 8H,4L &- &- & 7.80 & 5.47 & 3.86 & 2.84 & 0.7 & 2 & 4.3 & 7.5 \\
    Reuse & 4 & 2 & 8.25 & 5.83 & 4.37 & 3.24($13.9\%$) & 0.57 & 1.7 & 3.6 & 6.2($17.3\%$) \\
    Reuse & 4 & 3 & 8.52 & 5.72 & 4.41 & 3.28($15.4\%$) & 0.47 & 1.62 & 3.52 & 6.17($17.7\%$) \\
    \bottomrule
  \end{tabular}
\end{table}
}

\newcommand{\insertTableVIT}{
\begin{table}[ht]
  \caption{\textbf{ViT.} ImageNet finetuning accuracy of ViT models pretrained on JFT-300M. }
  \label{tab:vit_results}
  \centering
  \begin{tabular}{lcccccccc}
    \toprule
     \bf Model & \bf Reuse & \bf Reuse  &\bf Params &\bf FLOPS  & \bf Steps/Sec & \bf ImageNet Top-1  \\
     & heads ($K$) & layers ($P$) & & & & Acc\\ \midrule
     ViT - 12H,12L & - & - & 1.0 & 1.0 &  4.02 & 84.51$\pm$0.13 \\
     Reuse & 6 & 10 & 0.93 & 0.92 & 4.28 & 84.29$\pm$0.06 \\
     Reuse & 12 & 6 & 0.92 & 0.90 &  4.51 & 83.55$\pm$0.18 \\
     Reuse 13L & 6 & 10 & 1.01 & 0.99 &3.96 & \CC{84.69$\pm$0.24} \\
    \bottomrule
  \end{tabular}
\end{table}
}

\newcommand{\insertTableTranslate}{
\begin{table}[ht]
  \caption{\textbf{Machine Translation.} Median translation performance (BLEU scores) on Newstest2018 dataset. We notice that reusing attention results in similar performance while saving on the computation and parameters.}
  \label{tab:translate_results}
  \centering
  \small
  \begin{tabular}{lcccccccc}
    \toprule
     \bf Model & \bf Reuse & \bf Reuse  &\bf Params &\bf FLOPS & \bf EN-DE  & \bf DE-EN & \bf EN-CS &\bf CS-EN \\
     & heads ($K$) & layers ($P$) & & &  &  & &  \\ \midrule
     Baseline - 8H,6L & - & - & 1.0 & 1.0 & $39.22\pm0.24$ & $38.37\pm0.09$ & $18.27\pm0.13$ & $23.18\pm0.12$ \\
     Reuse & 2 & 4 & 0.91 & 0.96 & \CC{$39.20\pm0.17$} & $\CC{38.41\pm0.21}$ & $\CC{18.53\pm0.16}$ & $\CC{23.49\pm0.17}$ \\
     Reuse & 4 & 4 & 0.90 & 0.92 & \CC{$39.05\pm0.22$} & \CC{$38.20\pm0.14$} & $\CC{18.16\pm0.19}$ & $\CC{23.32\pm0.11}$ \\
\bottomrule
  \end{tabular}
\end{table}
}

%% file: 1_abstract.tex
\begin{abstract}
Pairwise dot product-based attention allows Transformers to exchange information between tokens in an input-dependent way, and is key to their success across diverse applications in language and vision. However, a typical Transformer model computes such pairwise attention scores repeatedly for the same sequence, in multiple heads in multiple layers. We systematically analyze the empirical similarity of these scores across heads and layers and find them to be considerably redundant, especially adjacent layers showing high similarity. Motivated by these findings, we propose a novel architecture that reuses attention scores computed in one layer in multiple subsequent layers. Experiments on a number of standard benchmarks show that reusing attention delivers performance equivalent to or better than standard transformers, while reducing both compute and memory usage.
\end{abstract}

%% file: 2_intro.tex
\section{Introduction}
Multi-head dot product attention is a key component of Transformer models~\citep{vaswani2017attention}. Each head in each Transformer layer allows aggregating information across different groups of tokens in a sequence, with the aggregation pattern in each head being input-dependent through the use of dot products between the learned query and key projections of token representations. This approach has proven empirically to be remarkably successful, with Transformers delivering state-of-the-art performance across a number of applications in language, vision, and beyond. Given this success, a considerable amount of research has focused on analyzing and interpreting the attention scores computed by trained Transformer models~\citep{ serrano2019attention, jain2019attention, wiegreffe2019attention, clark2019does, rogers2020primer} to gain insight into how attention is useful for inference in specific applications.

Our paper begins with analysis of a different aspect of attention---namely, the redundancy of attention scores computed by a Transformer model. Recently, \citet{bhojanapalli2021eigen} analyzed the variability in attention scores computed from different inputs sampled from a typical data distribution. In this work, we instead focus on the variability of scores across different heads in different layers that are computed for the same input. We perform a systematic analysis looking at the similarities in attention scores computed in different layers, after matching the closest heads for a given pair of layers. These similarities are computed for each typical input example on trained language and vision Transformer models, and then aggregated across a large set of examples from the corresponding training set.

Surprisingly, we find a high degree of similarity between different layers, with adjacent layers in a model exhibiting the most similarity. This suggests that although a standard Transformer model recomputes attention scores multiple times, much of this computation is redundant: the number of \emph{distinct} attention scores used for aggregation is much smaller than the total number of heads across all layers of a typical model. We also show that this is not an inherent characteristic of the Transformer architecture (random models do not produce similar attention scores), and that therefore this redundancy results from the structure of the problem domain.

Motivated by this analysis, we propose the \emph{Reuse Transformer}: a modified Transformer architecture which saves on redundant attention computation to deliver reductions in both compute and memory usage. As illustrated in Figure~\ref{fig:reuse}, a Reuse Transformer layer uses a mix of standard exact computation heads with ``reuse heads'', that borrow attention scores computed in previous layers instead of computing their own via dot products of query-key projections. This Reuse layer can be included in standard Transformer models with different configurations: either by reusing a fraction of the heads in most layers, or by reusing all heads in a fraction of the layers.

Since reuse heads do not use their own query and key projections, this leads to a reduction in the number of model parameters (and associated memory needed to save these parameters and their gradient moments). But unlike other approaches that also reduce parameters by sharing query-key projection parameters in multiple layers~\citep{dehghani2018universal, lan2019albert}, the Reuse Transformer also saves on the actual attention computation---thereby reducing the computational cost during both training and inference, and the memory needed to store intermediate projections.

\insertFigReuse

We evaluate Reuse Transformers in a variety of different settings: including on both language and vision tasks, and on encoder-only models like BERT~\citep{devlin2018bert}, encoder-decoder models like T5~\citep{raffel2020exploring}, and on Vision Transformers (ViT)~\citep{dosovitskiy2020image}. Our experiments consider standard benchmarks where models are trained for specific tasks from scratch --- Machine Translation (WMT 2018)~\citep{rej2018findings} and the Long Range Arena (LRA)~\citep{tay2021long} ---as well as those involving pre-training and finetuning --- GLUE~\citep{wang2019glue}, SuperGlue~\citep{wang2019superglue}, SQuAD~\citep{rajpurkar2016squad} and ImageNet~\citep{deng2009imagenet}. In addition to task performance, we also benchmark wall clock training time and memory usage to confirm that reusing attention translates readily to real world resource savings.

Through this extensive evaluation, we show that reusing attention scores saves compute and memory while yielding equivalent (and sometimes better) performance compared to standard Transformer models. We also show that when models with reuse are augmented with additional layers, to match baseline in terms of parameters, they perform better. Thus Reuse Transformers deliver a better performance and thus provide a better trade-off between resource usage and performance, indicating that attention score reuse is a useful inductive bias in Transformers.

In summary, our contributions in this paper are as follows.
\begin{itemize}
  \item We systematically analyze the similarity of attention computed by different layers of standard Transformer models, and show that attention computed in different layers are very similar.
  \item We develop a novel architecture for Transformer models that reuses attention scores from earlier layers in subsequent layers, thus reducing compute and memory usage.
  \item We evaluate the proposed architecture on a wide variety of baseline models---BERT, T5, ViT--- and benchmarks---GLUE, SuperGlue, SQuAD, ImageNet, WMT, and LRA---showing both actual savings in training time and memory usage, and at the same time, equivalent or better performance than standard Transformer models.
\end{itemize}

%% file: 3_related.tex
\section{Background}

\subsection{Transformer}\label{sec:transformers}

A Transformer encoder layer has two components: 1) a multi-head self-attention layer and 2) a tokenwise feed-forward (MLP) layer. A Transformer decoder layer in addition has a cross-attention layer, with attention between output of the encoder and the input to the decoder.  The input to these models is a sequence of vectors that are usually embeddings of an input token sequence. We let $\mX \in \R^{d \times n}$ denote the input embedding matrix of sequence length $n$ with embedding size $d$. Note that we denote vectors and matrices with small ($x$) and capital ($\mX$) bold letters respectively in this paper. The self-attention layer then updates these embeddings by computing pairwise dot product attention between the input embeddings. Both attention and feed-forward layers use layer normalization and skip connections. 

The self-attention layer computes dot product based attention scores as
\begin{align}\label{eq:attention_scores}
    \mA_{\mZ} = \sigma  \left(\mZ^\top \mW_Q^\top \mW_K \mZ / \sqrt{d} \right),
\end{align}
 where $\mW$ are trainable parameter matrices, and $\mZ$ is the layer input. Here $\sigma$ is a column-wise softmax operator.  Projections $\mW_Q \mZ$ and $\mW_K \mZ$ are usually referred to as query and key projections respectively. The attention scores are then used to linearly combine token embeddings as follows: $\mY= \mA_{\mZ}  \cdot \mZ^\top \mW_V^\top \cdot \mW_0$, where $\mW_V$ and $\mW_0$ are referred to as value and output projections respectively. The output of the attention layer is fed into a tokenwise feedforward layer: $ \mW_2 \phi \left( \mW_1 \mY^\top \right)$, with $\phi$ being an elementwise non-linear activation function such as a ReLU or GELU.  
 
 Multi-head attention involves multiple such trainable attention heads in a single layer, whose outputs are concatenated before multiplication with a common $\mW_0$.

\subsection{Related Work}

Given the key role of the attention in Transformers, several works have analyzed the attention scores computed by these models. Many used probing tasks, such as syntax dependency and coreference resolution, to test the natural language understanding computed by attention layers~\citep{clark2019does, hewitt2019structural,voita2019analyzing,rogers2020primer}. \citet{clark2019does} found that heads in the initial layers of a BERT model tend to attend more broadly, while heads in later layers attend to specific tokens and perform better at language understanding tasks. Following these works, \citet{raganato2020fixed} explored replacing input dependent attention with a combination of fixed attention patterns and one learnable pattern per layer.

Another line of research~\citet{voita2019analyzing,michel2019sixteen} showed that one can prune away many heads in an attention layer of a BERT model, after training. This pruning removes redundant attention heads, but leads to savings only during inference and not training. Through a systematic analysis of the similarity of attention scores computed by different layers, our work proposes a novel architecture that reduces redundant attention computation, leading to compute and memory savings both during training and inference.

It is well known that the cost of attention computation grows quadratically with input sequence length and poses challenges for training Transformers for long sequence length tasks. Several works address this issue by proposing efficient transformers that compute approximations to pairwise dot product attention---e.g. sparse~\citep{child2019generating,kitaev2020reformer,zaheer2020big,yun2020} and linear~\citep{choromanski2020rethinking,peng2020random}. These models reduce the computation and memory usage significantly but typically under-perform standard Transformers. We refer the reader to the survey by \citet{tay2020efficient} for a more detailed discussion. Our work focuses on reusing standard attention computation, and improves efficiency while maintaining performance.

Recently, \citet{bhojanapalli2021eigen} analyzed the variability of attention scores computed by a pre-trained model across different inputs. Finding them to be low-rank, they proposed a partial computation-based approach. While it reduced the cost of attention computation, also led to a drop in performance. Our work focuses on the similarity of attention scores computed in different layers for the same input, leading to a novel efficient architecture that does as well as or better than standard Transformers.

%% file: 4_analysis.tex
\section{Attention Similarity Analysis}\label{sec:analysis}

We now present an analysis of the similarity of attention scores computed in different layers by Transformer models trained for both language and vision tasks.

\paragraph{Preliminaries.} Recall that each row of an attention score matrix is the output of a soft-max operator and thus lies on a probability simplex (all entries are non-negative and sum to one). Therefore, to compute the similarity between any pair of attention score matrices $\mA$ and $\mA'$, we use a metric based on the total variation (TV) distance between them as:
\begin{align}\label{eq:tv_similarity}
    \gS(\mA, \mA') = 1 - \frac{1}{n} \sum_{p=1}^{n} \textrm{TV}(\mA[p,:], \mA'[p,:]) = 1 - \frac{1}{n} \sum_{p=1}^{n} \frac{1}{2} \|\mA[p,:] - \mA'[p,:]\|_1.
\end{align}
Here $n$ denotes the query sequence length (\# rows of $\mA$), $\|. \|_1$ the $l_1$ norm, and $\mA[p,:]$ the $p$th row of the attention scores matrix $\mA$---corresponding to scores for the $p$th query. Since attention scores form a probability distribution for each query, the total variation distance is always between $0$ to $1$. Hence the similarity also lies in $[0, 1]$.

Using \eqref{eq:tv_similarity} as our metric, we analyze the redundancy in attention scores computed by heads in different layers of a trained Transformer model on typical inputs from a dataset. Given $T$ inputs, we pass each through a given model and let $\mA_{l,h}^t$ denote the scores computed for the $t^{th}$ example in the $h^{th}$ head in layer $l$. Since there is no natural alignment between heads in different layers of a model, we define similarity from a reference layer $l$  and head $h$ to a target layer $l'$, while finding the best target head in that layer. This similarity score $c_{(l, h), l'}$ is defined as:
\begin{align}\label{eq:closest_head}
c_{(l, h), l'} =  \max_{h'} \frac{1}{T}\sum_{t=1}^T \gS(\mA_{l, h}^t, \mA_{l', h'}^t).
\end{align}
Note that we compute the average similarity for each choice of the target $h'$ before computing the best head, i.e., designating a common choice for the target head for all examples.

We use the above definitions of similarity to analyze scores from two BERT~\citep{devlin2018bert} (12 layers and 12 heads, and 24 layers and 16 heads) and one Vision ViT~\citep{dosovitskiy2020image} (12 layers and 12 heads) model (please refer to Section~\ref{sec:experiments} for training details). The similarities are computed for scores from 10k examples from the Wikipedia and ImageNet datasets for the BERT and ViT models respectively. 

\insertFigBestSimilarityModels
\paragraph{Results.} 

We begin by looking at the similarity between all pairs of layers for all three models in Figure~\ref{fig:best_heads_model}, looking at the best matched head in this case---for every pair of layers $(l, l')$, we visualize $\max_h c_{(l, h), l'}$. We find a surprisingly high degree of similarity in attention scores computed in different layers for all the three models. In particular, we see that the similarity between adjacent layers is especially high. Since we consider the best head here, these results imply that there is at least one head that is redundant in all pairs of layers that have high similarity. We also note that the ViT model shows a greater degree of similarity than the BERT models.

\insertFigHeadsBertLarge

For a deeper understanding of similarity across different heads, we restrict ourselves to adjacent layers (where we find similarity to be the highest in Figure~\ref{fig:best_heads_model}), and plot the similarity for different heads in the source layer, not just the best head. We rank heads from lowest (rank 1) to highest similarity, and plot these between successive layers in Figure~\ref{fig:heads_bert_large} for the 24-layer BERT and 12-layer ViT models. We again notice that, for BERT, the best head (rank 16) has a high similarity of around $0.8$ between successive layers. While the worst head has a low similarity ($<0.5$), even rank 5 head has a high similarity of around $0.6$. This suggests that the majority of the heads in a layer compute similar attention as in the previous layer. Some heads do seem to compute novel attention scores that are different from earlier layers.

\paragraph{Role of the Problem Domain.}

A natural question to ask is whether the similarity we observed above is inherent to the Transformer architecture. We evaluate this question both theoretically and empirically. First, we prove that a pair of random heads is expected to produce significantly different attention scores, with low similarity.
\begin{lemma}[Random attention has less similarity]\label{lem:random_attn}
Let the entries of the query and key projection matrices of two heads $\mW_{q1}, \mW_{k1}, \mW_{q2}$ and $\mW_{k2}$ be i.i.d randomly with a zero mean distribution. Let $\tilde{\mA}_1 = \mX^T*\mW_{q1}^T*\mW_{k1}*\mX$ and $\tilde{\mA}_2 = \mX^T*\mW_{q2}^T*\mW_{k2}*\mX$ be the pre-softmax attention scores computed using these matrices for any fixed $\mX$. Then, $$\E[(\tilde{\mA}_1 - \tilde{\mA}_2)_{ij}^2] = 2 \E[(\tilde{\mA}_1)_{ij}^2] = 2 \E[(\tilde{\mA}_2)_{ij}^2].$$
\end{lemma}
This lemma, proved in \S\ref{sec:appx_analysis} of the appendix, implies that even if the token embeddings input to attention heads in two different layers were the same, we would expect to see a large difference in their attention scores.

\insertFigBestSimilarityRandom

Therefore, the similarity that we observe is because the training process converges to models that find it beneficial, for the given inference task, to recompute similar attention scores in subsequent layers. We verify this empirically as well, by comparing a $6$ layer BERT model trained on random data to the standard one trained on Wikipedia in Figure~\ref{fig:best_heads_random}. We evaluate the similarity scores computed by both models on both Wikipedia data (1st and 2nd column) and on random data (3rd and 4th column). We notice that the degree of similarity depends more on the model, i.e., the data it was trained on, than on the data on which the attention scores are being computed. This suggests that the redundancy in attention is a result of training the model for a natural language problem domain.

%% file: 5_model.tex
\section{Reuse Transformer}\label{sec:reuse}

Attention score computation is one of the more expensive operations in a Transformer layer, since it scales quadratically with sequence length. Based on our observation that attention scores computed in different layers of a Transformer model are often redundant, we propose a novel Transformer layer which reuses attention scores computed in previous layers rather than computing them through a dot product of query-key projections. 

\paragraph{Reuse Multihead Attention.} We now describe the Transformer model with our proposed reuse layers in Algorithm~\ref{alg:reuse}. The model is also visually illustrated in Figure~\ref{fig:reuse}. In addition to the standard specification of the total number of layers $L$ and heads per layer $H$, the Reuse Transformer architecture also depends on the choice on the number of layers in which to employ reuse $P <L$, and the number of heads $K < H$ to be reused in every such layer. Note that Alg.~\ref{alg:reuse} describes only the modified attention computation. The remaining steps---query-key projection for exact heads, value projections for all heads and combination with attention scores, feed-forward layers, etc.---are the same as in standard Transformers as described in \S~\ref{sec:transformers}. Further note that this mechanism can also be used for cross-attention layer in the decoder models~\citep{vaswani2017attention}, by interpreting the exact and reuse attention scores in Alg.~\ref{alg:reuse} to refer to the ones from the cross-attention layer.

\begin{algorithm}[!th] 
    \caption{\pmb{Reuse MultiHead Attention}}
	\begin{algorithmic}[1]
	    \State Given: \# Layers $L$ and heads $H$, and reuse layers $P < L$ and reuse heads $K \leq H$.
	    \State Layer 1: Compute attention scores $\mA_{1, h}, \forall h \in [1, H]$. \Comment{First layer is always exact.}
	    \State Set Reuse attention scores $\mR_{1} = [\mA_{1, h \in [1, H]}].$
	    \For{$l = 2, \cdots, P+1$} \Comment{Reuse $K$ heads in the next $P$ layers.}
	        \State Compute attention scores only for $H-K$ heads  $\mA_{l, h \in [1, H-K]}$.
	        \State Reuse attention scores for $K$ heads $\mA_{l, h \in [H-K+1, H]} = \mR_{l-1, h \in [1, K]}$.
	        \State Set Reuse attention scores $\mR_l = [\mA_{l, h \in [1, H-K]}, \mA_{l, h \in [H-K+1, H]} ].$
    	\EndFor
        \For{$l = P+2, \cdots, L$} \Comment{Remaining $L-P-1$ layers are exact.}
	        \State Compute attention scores for all heads $\mA_{l, h}, \forall h \in [H]$.
    	\EndFor
	\end{algorithmic}
	\label{alg:reuse}
\end{algorithm}

We begin by noting from Alg.~\ref{alg:reuse} that attention is computed exactly for all heads in the first layer---since the first layer does not have access to any previous attention scores to reuse---and that we choose to use reuse layers in the first $P$ layers after the first one. Moreover, although we reuse attention scores, we still carry out the remaining steps of a self-attention layer, namely value projection and combining these values weighted by the attention scores. For ablation against variants that reuse different layers and skip attention computation entirely, please refer to \S~\ref{sec:reuseconfigs} in the appendix.

Every layer $l$ in our model outputs both the updated token embeddings as well as a set of $H$ attention scores $\mR_l$. In every layer that is reused, we compute exact attention scores for the first $H-K$ heads. For the remaining $K$ heads, we copy over scores from the first $K$ heads from $\mR_{l-1}$. The layer then passes on this new set of $H$ attention scores $\mR_l$ as input to layer $(l+1)$. Note that since we stack the new exact attention heads to the top of $\mR_l$ and retain scores from the first few heads of $\mR_{l-1}$, every reuse layer effectively uses the set of most recently computed $H$ scores.

Reusing attention scores in this way is a reasonable approximation given our empirical observation that attention scores in trained models are redundant. We analyze this formally in the context of a simplified two-layer linear Transformer model (similar to \citep{NEURIPS2020_ff4dfdf5}), and show that in the presence of high attention similarity, reusing attention generates similar outputs. We present a informal version of our result here, and include the more complete statement and proof in \S~\ref{sec:appx_analysis} in the appendix.
\begin{lemma}[Informal version of Lemma~\ref{lem:reuse_error}]
Let attention computed by the two layers be $\epsilon$ apart for all inputs $\mX$, i.e $ \|\mA_1 -\mA_2\| \leq \epsilon$. Then, there exists a reuse Transformer, such that the error in the outputs scales as $O(\epsilon).$
\end{lemma}

\paragraph{Configurations.} The parameters $P, K$ control how much attention computation is reused in the Transformer model, reducing the number of attention computations from $L*H$ heads to $(L*H - P*K)$. Note that for a given reuse budget ($P*K$) there are many ways of choosing parameters $P$ and $K$. We consider two different settings for our experiments. 

\begin{itemize}
\item \textbf{Partial layer Reuse.}
In this setting, we always set the number of reuse layers $P$ to be $L-2$ and vary $K$, such that all heads of the first and last layer compute attention scores, and rest of the layers reuse $K$ heads. In this architecture, every layer has atleast one head (when $K < H$) that computes attention scores.
\item \textbf{Full layer Reuse.}
In this setting we always set $K$ to be $H$ and vary $P$, i.e., attention is not computed in $P$ layers of the model and is reused from the earlier layer. Note that we need to again set the first layer exact to be able to reuse the attention scores in the following layers.
\end{itemize}

\paragraph{Computational Complexity.} Reusing the attention scores reduces both memory and computation of attention layer as heads that reuse attention scores do not have to compute the query and key projections as well. Thus the model reduces the attention score computation cost in each layer from $H\cdot n^2$ to $(H-K) \cdot n^2$, for input sequence length $n$ with $K$ heads being reused. This reduces the overall computational complexity of the multihead attention layer from $4\cdot d^2 \cdot n + 2 \cdot d \cdot n^2$ to $(1 -  \frac{K}{2H}) \cdot \left[4\cdot d^2 \cdot n + 2 \cdot d \cdot n^2\right]$. Similarly this reduces the number of parameters from $4 \cdot d^2$ to $(1 -  \frac{K}{2H}) \cdot 4 \cdot d^2$. 

%% file: 6_experiments.tex
\section{Experiments}\label{sec:experiments}

In this section we present experiments to show the advantage of reusing attention scores in reducing computational costs while matching or improving the performance of Transformers. We consider two different settings for our experiments 1) pre-training followed by finetuning, 2) training from scratch. For the first setting we consider BERT~\citep{devlin2018bert}, T5~\citep{raffel2020exploring} and ViT~\citep{dosovitskiy2020image} models. For the second setting we consider Machine Translation on WMT2018~\citep{rej2018findings} and the Long Range Arena (LRA) benchmark developed to test Transformers on long sequence length tasks in multiple domains~\citep{tay2021long}. Note that our setup includes encoder-only as well as encoder-decoder models. In both these settings we will see that reusing attention not only reduces computation but can also improves performance in some settings. Moreover, we will show that, one can always achieve better performance by matching the number of parameters of the reuse Transformer with the standard Transformer. We use the publicly available implementations, with the \textbf{same hyperparameters} for all the experiments for a task, and report them in detail in Appendix~(\S~\ref{sec:appx_setup}). 

\subsection{Experimental Setup}
We first describe our experimental setup for all the tasks.

\textbf{BERT}. BERT models are Transformers pre-trained with Masked Language Modeling (MLM) objective on Wikipedia and Books datasets~\citep{devlin2018bert}. These are encoder only models with bi-directional attention across input tokens. We follow a similar pre-training and finetuning recipe as BERT. We report the finetuning results on the MNLI~\citep{mnli} and SQuAD V1.1/V2.0~\citep{rajpurkar2016squad} tasks in Table~\ref{tab:bert_results}. 

We consider two models i.e., $\BB$ with 12 layers, 12 heads and $\BL$ with 24 layers, 16 heads. We consider two different settings for reusing attention scores: 1) partial layer---reusing $6/8$ heads per layer, 2) full layer---reusing scores in the beginning $6/12$ layers (Algorithm~\ref{alg:reuse}) in the $\BB$ and $\BL$ models respectively. Note that for both the settings the first layer computes attention scores for all heads as described in Section~\ref{sec:reuse}. Further we find it useful to have the last layer also compute attention scores for all heads in the first setting. In addition we also consider reuse transformers with more layers (Reuse 13L and 26L), that match the parameters of the baseline models.

\textbf{T5.} T5 models, unlike BERT, are encoder-decoder Transformer models that are pretrained on a similar objective as BERT models but on the C4 dataset~\citep{raffel2020exploring}. With a unified text to text framework, these models are shown to generalize easily to a variety of finetuning tasks. We consider the finetuning tasks from GLUE~\citep{wang2019glue} and SuperGLUE~\citep{wang2019superglue} benchmarks for our experiments. Similar to the BERT experiments we consider models of two different sizes---$\TB$ and $\TL$ with $12$ and $24$ layers respectively per encoder and decoder. Note that we reuse the attention in both the encoder and decoder, including in the cross attention layer. 

\textbf{ViT.} We next consider experiments with {Vision Transformers}~\citep{dosovitskiy2020image}. These models are pre-trained on JFT-300M~\citep{sun2017revisiting}, a dataset with 300M images and around 375M labels, and are finetuned on ImageNet. We consider the ViT-Base model, a 12 layer Transformer model with 12 heads per layer. For finetuning we use a resolution of 384x384 and a patch size of 16 resulting in a sequence length of 576 tokens per image. We consider again both partial and full layer reuse settings. 

\textbf{Machine Translation.} We now consider the experiments on Machine Translation. For training we use two language pairs in both directions from WMT 2018~\citep{rej2018findings}---English, German (en-de and de-en), and English, Czech (en-cs and cs-en). We report the test performance on Newstest 2018 datasets computed using the SacreBLEU~\citep{post2018call}. We use a baseline Transformer model with 6 layers per encoder and decoder with 8 attention heads per layer following~\citet{vaswani2017attention}.

\textbf{LRA.} Finally, we consider the {Long Range Arena} benchmark~\citep{tay2021long} developed to test Transformers on long sequence length tasks with input sequence lengths ranging from 1k to 8k. It consists of five tasks covering logical, textual and vision data. All tasks use a 4 layer Transformer model with the exception of Image classification task, that uses a 1 layer model. Hence, we do not report results on this task. We consider three different reuse settings.

\subsection{Results}
We now present our experimental results.

\setlength{\tabcolsep}{4pt}
\insertTableBert
\textbf{BERT.} We report the median finetuning results over 3 independent runs in Table~\ref{tab:bert_results}. In addition to pretraining and finetuning metrics, we also report the relative computation cost required for each model. We first notice that models that reuse attention scores require lesser computation and parameters. Interestingly reusing attention scores results in similar performance on both pre-training and finetuning tasks. Further partial and full layer reuse have similar average performance. Finally the $13$ and $26$ layer reuse models achieve the best performance while having the same number of parameters as the $\BB$ and $\BL$ models respectively.

\insertTableTfive
\textbf{T5.} We again report the median finetuning results over 3 independent runs on the GLUE and SuperGLUE benchmarks in Table~\ref{tab:tfive_results}. We report the average scores across all tasks. We first notice that, interestingly, reusing attention results in better performance over $\TB$ and matches performance of $\TL$, while saving on computation resources. This is observed for both forms of reusing attention. In addition to model parameters and FLOPS we also report the training wall clock time, in terms of steps per second for these models on TPUv3 with 32 chips for $\TB$ and 64 chips $\TL$ respectively. We notice that reusing attention leads to substantial speedups while improving the performance. Finally we also consider 13/26 layer reuse models that have the same number of parameters as the baseline, and achieve the best performance.

\insertTableVIT
\textbf{ViT.} We present the results in Table~\ref{tab:vit_results}. We notice again that reusing attention leads to similar performance as baseline with reduction in parameters and compute. Interestingly, partial reuse has better performance over full layer reuse. Further increasing the reuse model size (13L), to match the baseline size, results in an improved performance.

\insertTableTranslate
\textbf{Machine Translation.} We report the median BLEU scores in Table~\ref{tab:translate_results}. We see again that reusing attention scores leads to similar performance as the baseline while saving on computational resources.

\insertTableLRA
\textbf{LRA.} We notice in Table~\ref{tab:lra_results} that reusing attention scores leads to better performance for the tasks in LRA benchmark. Interestingly we see better performance when more heads/layers reuse attention scores. We attribute this to the regularization effect of reusing the attention scores. Note that since some tasks only use a model with only 4 attention heads, the results for reusing 8 attention heads are the same as the ones with 4 attention heads.

\subsection{Ablation}
\insertFigAblation
In this section we present ablation results with the reuse Transformers (Figure~\ref{fig:ablation}). We first train the reuse Transformer that reuses 6/12 attention layers ($P=6, K=12$) for the same amount of time as the baseline and notice significant performance gains ($>1\%$) over the $\TB$ model on the SuperGLUE benchmark (left). We next present an ablation varying the number of reuse heads ($K$), with $P=10$, for the 12 layer $\TB$ (middle) and ViT (right) models. We notice that reusing a few attention heads improves performance. Interestingly, even when reusing all the heads, the drop in performance is not much for the ViT models.

\subsection{Computational savings}
\insertTableLRAPerformance
In this section we present comparison of models compute (steps per second) and memory usage during training for different sequence lengths on the text classification task in the LRA benchmark in Table~\ref{tab:lra_performance}. We use a batch size of 32 and train on TPUv3 with 16 chips. The baseline model has 4 layers with 4 attention heads per layer. We consider two different reuse settings. Both the settings show improvement in both speed and memory usage over the baseline with the best model achieving $15.4\%$ speedup and $17.7\%$ reduction in memory usage.

%% file: 7_appendix.tex
\section{Experimental setup}\label{sec:appx_setup}
In this section we present details of our experimental setup. For all the models we closely follow the settings from their original implementation and use the publicly available code. We use TPUv2 and TPUv3s for our experiments.

\subsection{BERT}
We pretrain the BERT models on Wikipedia and Books datasets using the masked language model objective. We use a batch size of 512 and train for 1M steps using the Adam optimizer with $1e-4$ weight decay. We use dropout of $0.1$. For finetuning on MNLI, we use a batch size of 128 and train for 3 epochs. We use $3e-5$ learning rate. For finetuning on SQuAD, we use a batch size of 48 and train for 2 epochs. We use $8e-5$ learning rate. 

\subsection{T5}
We pretrain the T5 models on C4 dataset using the span corruption objective. We use a batch size of 128 with 512 input sequence length and 114 target sequence length. We pretrain for 524288 steps. We use Adafactor optimizer with $1.0$ peak learning rate. We use linear learning rate warmup for 10k steps followed by square root decay. We use dropout of $0.1$. For finetuning we train for additional 262,144 steps. We use a constant learning rate of $1e-3$. We use a target length of 84 and 62 for Glue and SuperGlue respectively.

\subsection{ViT}
We pretrain the ViT models on the JFT-300M dataset for 7 epochs. We use a batch size of 4096. We use Adam optimizer with a learning rate of 8e-4. We use 10k linear warmup steps and linear decay to 1e-5. We use a weight decay of 0.1. For finetuning on ImageNet, we train for 8 epochs with a batch size of 512. We use momentum SGD with a peak learning rate of 3e-2. We use 500 warmup steps and decay learning rate to 3e-4 using cosine schedule.

\subsection{Machine Translation}
We train the encoder-decoder Transformer on the WMT2018 datasets. We use a 6 layer model with a hidden size of 512 with 8 heads per layer. We use input sequence length of 256 with a batch size of 4096 with padding. We use Adam optimizer with linear warmup for 8k steps and square root decay.  We use the Tensor2Tensor framework with default settings for training all the models~\citep{vaswani2018tensor2tensor}.

\subsection{LRA}
We use the publicly available code for training models on this benchmark \footnote{https://github.com/google-research/long-range-arena}. We train all the models with the default configs. We use a Transformer with 4 layers and 8 attention heads per layer with a hidden size of 512. We use Adam optimizer with linear warmup and square root decay. We use dropout and a weight decay of 1e-1. 

\section{Additional experimental results}\label{sec:appx_experiments}

\subsection{Different configurations of Reuse}
\label{sec:reuseconfigs}
\begin{table}[ht]
  \caption{Comparison of different configurations of reusing attention.}
  \label{tab:bert_skip}
  \centering
  \begin{tabular}{lccccc}
    \toprule
     \bf Model  & \bf Reuse/skip & \bf MLM  & \bf MNLI & \bf SQuAD V1.1 &\bf SQuAD V2.0 \\ 
     & layers (P) & Acc & Acc & F1 & F1 \\ \midrule
     $\BL$ & - & 73.76 & 87.97 & 91.87 & 82.41 \\
     Reuse - Proposed & 12 & 73.64 & {87.75} & {91.92} & {82.56} \\
     Reuse - Alternate & 12 & 73.88 & 87.31 & 91.56 & 82.29 \\
     Reuse - AllEnd & 12 & 74.07 & {87.24} & {91.55} & {81.7} \\
     Skip & 12 & 73.4 & 87.3 & 91.65 & 81.83 \\
     \bottomrule
  \end{tabular}
  \vspace{-0.05in}
\end{table}

In this section we present comparisons between different configurations of reuse attention. We consider full layer reuse setting. The proposed architecture in Algorithm~\ref{alg:reuse}, Reuse-Proposed, reuses the first $P$ layers following the first layer. Here we consider three different variants, 1) Reuse - Alternate : here we reuse attention in alternate layers, i.e. we reuse attention in all the even numbered layers, with first layer being exact. 2) Reuse - AllEnd : here we reuse attention in the final $P$ layers with last layer being exact, i.e., we reuse attention in layers $L-P-1$ to $L-1$. 3) Skip - in this architecture we skip the attention computation completely in $P$ layers. The Transformer blocks in the Skip layers only consist of the tokenwise feedforward layers.

We compare above configurations using $\BL$ (24 layers, 16 heads) as baseline in Table~\ref{tab:bert_skip}. We first notice that the Reuse-Proposed performs the best among all the configurations, with Reuse-Alternate performing slightly better than the Reuse-AllEnd and Skip, both of which have the worst performance.

\subsection{Number of examples}
\insertFigBestSimilarityExamples
We use 10k examples to compute the mean attention similarity for results in Section~\ref{sec:analysis}. Since we are only computing mean Total Variation similarity of probability distributions in $512$ dimensions, $10K$ examples is enough. We also present mean similarity computed using $1M$ examples in Figure~\ref{fig:best_heads_6l_examples} and we got same results.

%% file: 8_analysis.tex
\newcommand{\dX}{\Delta\mX}
\section{Analysis}\label{sec:appx_analysis}
\subsection{Attention similarity}
In this section we present our analysis that shows that attention computed with random weights leads to less similarity in attention scores.

\begin{proof}[Proof of Lemma~\ref{lem:random_attn}]
Recall that $\tilde{\mA}_1 = \mX^T*\mW_{q1}^T*\mW_{k1}*\mX$ and $\tilde{\mA}_2 = \mX^T*\mW_{q2}^T*\mW_{k2}*\mX$. Let $\mX_i$ be the $i$th column of $\mX$.
Now we write the expected error between attention scores for a single entry.
\begin{align*}
    \E[\left(\mX^T*\mW_{q1}^T*\mW_{k1}*\mX \right.&-\left. \mX^T*\mW_{q2}^T*\mW_{k2}*\mX \right)_{ij}^2] \\&= \E[ (\mX^T*\mW_{q1}^T*\mW_{k1}*\mX)_{ij}^2 + (\mX^T*\mW_{q2}^T*\mW_{k2}*\mX)_{ij}^2 \\ &  \quad \quad \quad \quad- 2(\mX^T*\mW_{q1}^T*\mW_{k1}*\mX)_{ij}(\mX^T*\mW_{q2}^T*\mW_{k2}*\mX)_{ij}] \\
    &= \E[ (\mX_i^T*\mW_{q1}*\mW_{k1}^T*\mX_j)^2] + \E[(\mX_i^T*\mW_{q2}*\mW_{k2}^T*\mX_j)^2] \\ & \quad \quad \quad \quad - 2\E[(\mX_i^T*\mW_{q1}*\mW_{k1}^T*\mX_j)(\mX_i^T*\mW_{q2}*\mW_{k2}^T*\mX_j)] \\
\end{align*}
Now the last term in the above inequality is 0 as $\mW_{q1}, \mW_{k1}$ are independent from $\mW_{q2}, \mW_{k2}$ and have zero mean. Hence,
\begin{align*}
 \E[\left(\mX^T*\mW_{q1}^T*\mW_{k1}*\mX \right.&-\left. \mX^T*\mW_{q2}^T*\mW_{k2}*\mX \right)_{ij}^2] \\&=
     \E[ (\mX_i^T*\mW_{q1}*\mW_{k1}^T*\mX_j)^2] + \E[(\mX_i^T*\mW_{q2}*\mW_{k2}^T*\mX_j)^2] \\&=
     2 \E[ (\mX_i^T*\mW_{q1}*\mW_{k1}^T*\mX_j)^2] 
\end{align*}
The last equality follows since $\mW_{q1}, \mW_{k1}$ have identical distribution with $\mW_{q2}, \mW_{k2}$.
\end{proof}

\subsection{Reuse Attention}
In this section we present our analysis showing when can reuse attention approximate standard attention well. We consider a 2 layer 1 head Transformer architecture with simplifications for analysis. In particular following \citet{NEURIPS2020_ff4dfdf5} we consider an architecture that excludes the ReLU activation and layer-norm operation. Note that, unlike~\citet{NEURIPS2020_ff4dfdf5}, we allow for the softmax normalization. We emphasize that while these simplifications do affect performance they still preserve the main self-attention functionality - which is our main focus. Please see the discussion in \citet{NEURIPS2020_ff4dfdf5} for justification.

Additionally, we make the following approximation about the input to the attention scores in the second layer of the Transformer model.

Recall $\mA_2 =\mZ^T\mW_{q2}^T*\mW_{k2}*\mZ,$ where $\mZ$ is the output of the first layer $\mZ = \mX +\dX$, where $\dX = \mA_1 \mX \mW_1$. Hence,
\begin{align*}
    \mA_2 &= (\mX +\dX)^T * \mW_{q2}^T*\mW_{k2} * (\mX +\dX) \\
    &= \mX^T \mW_{q2}^T*\mW_{k2} \mX + \dX^T \mW_{q2}^T*\mW_{k2} \mX + \mX^T  \mW_{q2}^T*\mW_{k2}\dX + \dX^T \mW_{q2}^T*\mW_{k2}\dX
    \end{align*}
Note that usually $\|\dX\|$ is much smaller than $\|\mX\|$. Hence, we approximate the attention computation equation by ignoring the $\dX$ terms, giving us $\mA_2 = \mX^T * \mW_{q2}^T*\mW_{k2} * \mX$. We note that this is a reasonably approximation as we mainly consider attention in the first two layers in our analysis.

Under these simplification we can write the output of a 2 layer 1 head attention architecture as follows. Let input be a $\mX \in \mathbb{R}^{n \times d}$. Output of this model is 
\begin{align*}
    \mY = \mX + \mA_1\mX\mW_1 + \mA_2\mX\mW_2 +  \mA_2\mA_1\mX\mW_1\mW_2.
\end{align*}
Note that the different linear projections in the Transformer, value output, feedforward layer, all can be absorbed into a single projection $\mW$ as they are all linear projections. Here $\mA_1 = \sigma(\mX^T*\mW_{q1}^T*\mW_{k1}*\mX)$ and $\mA_2 = \sigma(\mX^T*\mW_{q2}^T*\mW_{k2}*\mX)$ are the attention scores computed by the two layers.

Under the same setting, the output of the reuse attention model is
\begin{align*}
    \widehat{\mY} = \mX + \widehat{\mA} \mX\mW_3 + \widehat{\mA} \mX\mW_4 +  \widehat{\mA}^2\mX\mW_3\mW_4.
\end{align*}
Here $\widehat{\mA} = \mX^T*\widehat{\mW}_{q}^T*\widehat{\mW}_{k}*\mX$ is the attention scores computed by the first layer and reused in the second layer. Let $\|.\|_2$ denote the spectral norm of a matrix.
\begin{lemma}\label{lem:reuse_error}
Let attention computed by the two layers be $\epsilon$ apart for all inputs $\mX$, i.e $$ \|\mA_1 -\mA_2\|_2 \leq \epsilon. $$ Also let all input and parameter norms to be less than 1, i.e. $\|\mW_1\|_2, \|\mW_2\|_2, \|\mX\|_2 \leq 1$. Then, there exists a choice for $\widehat{\mA}, \mW_3, \mW_4$, such that the error scales as $$\|\widehat{\mY} - \mY\|_2 \leq  2\epsilon + \frac{\epsilon^2}{2}.$$ 
\end{lemma}
This lemma shows us that if the attention scores are closer, then reuse attention the output remains close to the standard transformer output. Hence, this shows that we can reuse attention when there is high attention similarity without much output error. 

\begin{proof}
We present a construction based proof. We set $\mW_3=\mW_1$, $\mW_4=\mW_2$, and
$\widehat{\mW}_{q}^T*\widehat{\mW}_{k} = \frac{\mW_{q1}^T*\mW_{k1}+\mW_{q2}^T*\mW_{k2}}{2}$. Note that we can always find such matrices $\widehat{\mW}_{q}, \widehat{\mW}_{k}$ as they are all full dimensional and hence can be full rank. This implies $\widehat{\mA}=\frac{\mA_1 + \mA_2}{2}$ for all $\mX$.

\begin{align*}
\widehat{\mY} - \mY &=  (\mA_1-\widehat{\mA})\mX\mW_1 + (\mA_2-\widehat{\mA})\mX\mW_2 + (\mA_2\mA_1- \widehat{\mA}^2)\mX\mW_1\mW_2 \\
    &= \frac{(\mA_1 - \mA_2)}{2}\mX\mW_1 + \frac{(\mA_2 - \mA_1)}{2}\mX\mW_2 + (\mA_2\mA_1- \widehat{\mA}^2)\mX\mW_1\mW_2 \\
    &= \frac{(\mA_1 - \mA_2)}{2}\mX(\mW_1-\mW_2) + \frac{1}{2}(\mA_2\mA_1 -\mA_1\mA_2 - \frac{1}{2}(\mA_1-\mA_2)^2)\mX\mW_1\mW_2
\end{align*}
\begin{align*}
    \|\tilde{\mY} - \mY\|_2 &\leq \|\frac{(\mA_1 - \mA_2)}{2}\mX(\mW_1-\mW_2)\|_2 + \|\frac{1}{2}(\mA_2\mA_1 -\mA_1\mA_2 - \frac{1}{2}(\mA_1-\mA_2)^2)\mX\mW_1\mW_2\|_2 \\
    &\leq \frac{\epsilon}{2}\|\mX(\mW_1-\mW_2)\|_2 + \frac{\epsilon^2}{2}\|\mX\mW_1\mW_2\|_2 + \|\frac{1}{2}(\mA_2\mA_1 -\mA_1\mA_2) \mX\mW_1\mW_2\|_2  \\
    &=\frac{\epsilon}{2}\|\mX(\mW_1-\mW_2)\|_2 + \frac{\epsilon^2}{2}\|\mX\mW_1\mW_2\|_2 + \|\frac{1}{2}(\mA_1(\mA_1 -\mA_2)+(\mA_2-\mA_1)\mA_1) \mX\mW_1\mW_2\|_2  \\
    &\leq \frac{\epsilon}{2}\|\mX(\mW_1-\mW_2)\|_2 + \frac{\epsilon^2}{2}\|\mX\mW_1\mW_2\|_2 + \epsilon\|\mA_1\|_2\|\mX\mW_1\mW_2\|_2 \\
    &\leq 2 \epsilon + \frac{\epsilon^2}{2}.
\end{align*}
Note that $\mA_1$ is a stochastic matrix with max singular value at most 1.
\end{proof}